\documentclass[11pt]{article}

\usepackage[letterpaper,margin=1in]{geometry}
\usepackage[T1]{fontenc}
\usepackage[utf8]{inputenc}
\usepackage{lmodern}
\usepackage{microtype}
\usepackage{graphicx}
\usepackage{float}
\usepackage{amsmath,amssymb}
\usepackage{booktabs,array,tabularx}
\usepackage{placeins}
\usepackage{url}
\usepackage[hidelinks]{hyperref}
\hypersetup{pdftitle={An Amortized Efficiency Threshold for Comparing Neural and Heuristic Solvers in Combinatorial Optimization},pdfauthor={Sohaib Afifi}}
\graphicspath{{./}}

\title{An Amortized Efficiency Threshold for Comparing\\
Neural and Heuristic Solvers in Combinatorial Optimization}
\author{Sohaib Afifi\\
Univ. Artois, UR 3926,\\
Laboratoire de G\'enie Informatique et d'Automatique de l'Artois (LGI2A),\\
F-62400 B\'ethune, France\\
\texttt{sohaib.afifi@univ-artois.fr}}
\date{}

\begin{document}
\maketitle

\begin{abstract}
Neural optimization solvers require energy for training before they can solve new instances. We define the Amortized Efficiency Threshold (AET) as the number of instances needed to recover this training cost through energy savings relative to a fixed heuristic baseline. Both solvers must meet a stated solution-quality requirement. The threshold accounts for training energy and the energy used to solve instances, including the effect of processing them together in batches. We illustrate the framework through a case study of the capacitated vehicle routing problem, comparing different neural networks with hybrid genetic search on two problem sizes and two computing platforms. Neural solvers can recover their training cost against a heuristic given a long search budget. However, shorter heuristic runs that satisfy the same quality requirement remain less energy intensive than the admissible neural policies tested. The study shows how AET can guide deployment choices by making the effects of training cost, batch size and baseline choice explicit.
\end{abstract}

\section{Introduction}
\label{sec:introduction}

Neural combinatorial optimization solvers and classical heuristics spend energy on different schedules. A learned solver pays for training before deployment, then uses the same parameters to solve many instances. A heuristic usually avoids this initial learning cost but performs a search for every instance. A useful comparison must therefore account for both training and the number of problems eventually solved.

Consider a service that repeatedly solves routing problems with a fixed heuristic configuration. Replacing it with a neural solver saves computing energy only if the saving on future instances repays the energy assigned to training. The number of instances needed depends on how they are processed. Batching can reduce neural energy per instance, while a longer heuristic search can make the reference service more expensive. Both solvers must also meet the application's quality requirement: a fast solver returning unsuitable solutions is not an admissible replacement.

The Amortized Efficiency Threshold (AET) expresses this deployment question. It combines a declared training debt, measured energy per original problem instance, and a common quality requirement to identify the volume at which a neural policy reaches break-even against a fixed baseline. The fixed-cost calculation is elementary. The contribution is a framework for making that calculation meaningful, together with measurements showing how batch size, solution quality, and baseline choice affect its interpretation.

We illustrate the framework with the capacitated vehicle routing problem (CVRP), using 50 and 100 customers, two neural networks, and hybrid genetic search (HGS). The study includes independent training seeds on a Windows workstation and a Linux GH200 node. Against HGS given 10 seconds per instance, several neural policies repay their observed training cost. Shorter HGS runs that meet the same 5 percent quality requirement remain less energy intensive than every admissible neural policy measured. These findings describe two deployment decisions: replacing an existing solver configuration and comparing with a cheaper admissible alternative. The second decision tests the dependence on the baseline without changing the meaning of the first.

The main text develops AET and follows one numerical example through the routing study. The appendices provide the full protocol, statistical calculations, additional comparisons, and memory diagnostics. The CVRP results illustrate the framework; they do not establish an energy ranking for combinatorial optimization in general.

\section{Related work}
\label{sec:related}

Amortized optimization learns a reusable mapping from problem instances to solutions, replacing repeated optimization with prediction after an initial learning cost~\cite{amos2023amortized}. Pointer networks, attention-based policies, and POMO apply this idea to combinatorial problems and routing~\cite{vinyals2015pointer,bello2017neural,kool2019attention,kwon2020pomo}. Comparing these methods requires explicit budgets, common instances, and feasibility checks~\cite{accorsi2022guidelines}; generalization across routing distributions and sizes remains a separate issue~\cite{joshi2022learning}.

Green AI research motivates reporting energy, hardware, and experimental search alongside predictive performance~\cite{strubell2019energy,schwartz2020green,henderson2020systematic}. Luccioni et al.~\cite{luccioni2024power} estimate when accumulated inference energy equals the training and fine-tuning energy of the same model. A direct precedent for repayment relative to another solver appears in NeuralEF: Chatigny et al.~\cite{chatigny2023frontier} estimate that a learned portfolio-optimization model offsets its training emissions after about 4.2 billion evaluations relative to a CPU optimizer. Their training-carbon estimate was not measured during training. AET builds on the familiar break-even calculation by specifying a shared quality requirement, batch-dependent solving cost, and the accounting needed to interpret measured thresholds.

Our heuristic comparison uses PyVRP~\cite{wouda2024pyvrp}, an implementation of HGS for capacitated routing~\cite{vidal2012hgs,vidal2022hgs}. A solver name alone does not define its energy cost: the stopping rule and concurrency are part of the baseline. The study measures computing energy. Converting it to carbon or lifecycle impact would additionally require grid-intensity and hardware-manufacturing assumptions~\cite{lacoste2019quantifying,patterson2021carbon,gupta2021chasing,wu2022sustainable}.

\section{The AET framework}
\label{sec:framework}

Fix a neural policy $a$, including its trained checkpoint, decoding rule and numerical precision, and a heuristic policy $h$, including its search budget and concurrency. Let $D_a$ denote the training energy assigned to one deployed model. If a neural batch contains $b$ original problem instances and uses energy $E_a^{\mathrm{batch}}(b)$, its marginal energy and saving relative to the heuristic are
\begin{equation}
 e_a(b)=\frac{E_a^{\mathrm{batch}}(b)}{b},\qquad
 \Delta e_a(b;h)=e_h-e_a(b).
 \label{eq:saving}
\end{equation}
The denominator counts original instances. Multiple neural starts or geometric augmentations increase the work needed to solve an instance, not the number of instances served.

After $N$ instances, cumulative neural energy is $D_a+Ne_a(b)$, while heuristic energy is $Ne_h$. When $\Delta e_a(b;h)>0$, the continuous crossover and the first realizable volume in full batches are
\begin{equation}
 \widetilde N_a(b;h)=\frac{D_a}{\Delta e_a(b;h)},\qquad
 N_a^{\mathrm{full}}(b;h)=b\left\lceil\frac{D_a}{b\Delta e_a(b;h)}\right\rceil.
 \label{eq:aet}
\end{equation}
Batching matters primarily through the measured value of $e_a(b)$; rounding adds fewer than $b$ instances. Workloads with underfilled batches need their own energy estimate. In our 512-instance measurement corpus, batches of 96 or 192 end with a smaller batch, so their thresholds use whole corpus cycles instead (Appendix~\ref{app:statistics}).

An energy crossover is an AET only when both policies satisfy a stated quality requirement $\mathcal Q_\tau$. In the routing study, solutions must be feasible and have acceptable mean and tail cost gaps relative to an independent best-known reference. We require $q<\tau$, where $q$ combines the worst observed mean and 95th-percentile gaps with descriptive bootstrap upper bounds. Appendix~\ref{app:statistics} specifies the calculation. The reference is not an optimality certificate. The main examples use $\tau=5\%$ as an illustrative acceptance rule, not a universal routing standard.

The framework distinguishes three outcomes. A policy outside $\mathcal Q_\tau$ is inadmissible. If both policies pass but $\Delta e_a\leq0$, no deployment volume repays a positive training debt under the measured regime. If both pass and the saving is positive, Equation~\eqref{eq:aet} gives a threshold, provided the required energy quantities are known. When only a lower bound $D_a\geq D_a^{\mathrm{obs}}$ is available, the resulting threshold is also a lower bound. We mark these cases with a dagger. Their intervals describe variation conditional on observed debt and do not bound unobserved training energy.

The debt must match the decision. We use the mean measured cost of training one model under a specified recipe, rather than the sum over experimental replications. Architecture development, hyperparameter search, and evaluation outside the training windows are excluded from this deployment debt. Assigning a larger debt multiplies the continuous threshold proportionally. AET can therefore support different accounting choices, provided they are stated and used consistently.

The baseline is equally explicit. An AET against an existing HGS service answers whether replacing that service can save energy. We also examine a cheaper baseline selected on development data from a declared HGS family, then frozen before test evaluation. This sensitivity analysis does not redefine the fixed-baseline threshold or claim to find the globally least-energy solver.

\section{A routing case study}
\label{sec:method}

CVRP is used here as a worked application of AET. For each size, we use 256 development instances and a disjoint test set of 512 instances, with normalized Euclidean coordinates and integer customer demands. Reference costs are the best independently validated routes from three separate HGS runs. The timed baseline does not define its own gap. Returned routes are checked for customer coverage, vehicle capacity and Euclidean cost.

We compare two neural networks using the same pointer-decoder architecture and POMO training procedure. The Attention Model (AM) uses global self-attention over the depot and customer nodes. The model denoted GNN uses attention-based message passing on a sparse geometric graph, implemented with graph-transformer layers~\cite{shi2021unimp}. Each model is trained from scratch for 200 epochs of 20,000 instances, and the final checkpoint is evaluated without selection on test performance. We use five training seeds per size and encoder on Windows and four on GH200. Training energy covers initialization, optimization and checkpoint serialization. The resulting debt represents one trained model, averaged over those seeds. Full architecture and optimizer settings appear in Appendix~\ref{app:protocol}.

The common inference policy uses $n$ POMO starts and eight augmentations for an instance with $n$ customers, denoted POMO $n\times8$. We measure FP32 batches from 1 to 512, and BF16 at batches 64 and 256. Windows also covers greedy decoding and fewer augmentations. Quality is checked for each precision and decoding policy. Batch 64 is the common example for both sizes and platforms; it is not chosen by minimizing test energy. The two-point BF16 grid limits the corresponding CVRP100 inference conclusions.

HGS-10s uses a 10 second limit per instance and 16 single-threaded worker processes. Windows additionally varies the search budget and worker count, selecting the least-energy admissible configuration on development data. GH200 uses two policies fixed in advance: HGS-10s and short HGS runs of 10 iterations for CVRP50 or 30 for CVRP100, also with 16 workers. The latter comparison does not repeat the Windows baseline-selection search on GH200.

\begin{table}[t]
\centering
\caption{Platforms and energy boundaries used in the main comparisons. Neither boundary measures wall-plug energy. Runtime versions, sensors and excluded components are detailed in Table~\ref{tab:protocol}.}
\label{tab:measurement-summary}
\small
\begin{tabularx}{\textwidth}{@{}>{\raggedright\arraybackslash}p{0.16\textwidth}>{\raggedright\arraybackslash}X>{\raggedright\arraybackslash}X@{}}
\toprule
 & Windows workstation & Linux GH200 node \\
\midrule
Hardware & Intel i7-13700; RTX A4500 with 19.99 GiB physical VRAM & Grace/Hopper modules with GH200 144G HBM3e GPUs \\
Neural energy & CPU-package plus GPU counters & Reported power of one Grace/Hopper module, integrated over time \\
HGS energy & CPU-package counter, representing CPU-only deployment & Same module boundary, including idle GPU power \\
Replication & Five training seeds, five paired deployment rounds & Four training seeds, five deployment rounds per seed \\
\bottomrule
\end{tabularx}
\end{table}

Table~\ref{tab:measurement-summary} states the boundaries that must accompany the results. Windows compares CPU-plus-GPU neural energy with a CPU-only HGS deployment. GH200 counts one module for both, including the idle GPU during HGS. GNN trains on GH200 module 1 and deploys on module 0, while AM uses module 0 throughout. These differences prevent interpreting cross-platform thresholds as isolated hardware effects or the AM/GNN difference as a purely architectural effect. Windows GNN/CVRP50 and all GH200 debts except AM/CVRP50 are lower bounds; no unobserved energy is imputed.

Each deployment block repeats the common corpus for at least 120 seconds. We divide its energy by the number of original instances processed, including repeated passes, then aggregate these block ratios over the seed/round structure below. When $b$ divides 512, multiplying this rate by $b$ gives equivalent full-batch energy, rather than a separately metered batch. HGS energy is measured around the whole process pool, avoiding overlap between worker measurements. Reported time per instance is inverse aggregate throughput, not individual request latency.

Windows summaries use five seed/round pairs. GH200 first averages five rounds within each of four training seeds. Paired bootstrap intervals preserve the shared HGS measurements; GH200 resamples seeds and rounds separately. They are descriptive intervals, conditional on observed debt, and do not cover sensor calibration, unquantified host activity, or simultaneous uncertainty across the grid. The sampling and verification procedures are given in Appendices~\ref{app:statistics} and~\ref{app:reproducibility}.

For the main Windows batch curves, both tensor-allocation and allocator-reservation peaks must remain within physical GPU capacity in every round. The full records, including larger-than-VRAM measurements, remain in Appendix~\ref{app:memory}. These peaks do not prove page residency or explain every slowdown. GH200 uses the same batch grid, but its recorded memory peaks likewise do not establish exclusive device-memory residency.

\section{Results}
\label{sec:results}

The Windows AM/CVRP50 comparison gives a concrete AET calculation. With POMO $50\times8$ at batch 64 in FP32, AM uses about 6.60 J per instance, or 423 J per equivalent batch. HGS-10s uses 40.69 J per instance. The saving is therefore about 34.09 J per instance, and the one-model training debt is 4.62 MJ. Both policies meet the 5 percent quality gate. Using the unrounded measurements gives a continuous crossover of about 135,464 instances and a full-batch threshold of 135,488 instances. Its descriptive interval is approximately 134,000 to 137,000 instances. This is the volume that repays the declared training debt under this particular baseline and deployment policy.

Table~\ref{tab:main-results} extends the example to the other encoder, problem size and platform. GNN/CVRP50 has nearly the same solving energy as AM on Windows, but its larger observed training debt places its threshold near 146,000 instances. That estimate is a lower bound. Training costs also rise strongly with problem size: the measured 200-epoch debts range from about 1.3 to 1.4 kWh on Windows CVRP50 and from 7.8 to 8.3 kWh on CVRP100. GH200 means are about 1.1 and 5.3 to 5.4 kWh, under the different module boundary and with the lower-bound qualifications in Table~\ref{tab:debt}.

\begin{table}[t]
\centering
\caption{Common batch-64 examples against HGS-10s with 16 workers, using final-checkpoint POMO $n\times8$. All listed policies satisfy the 5 percent quality gate. Thresholds and descriptive 95 percent intervals are in thousands of instances. NS means no marginal saving. A dagger denotes a lower-bound threshold conditional on observed debt. Energy boundaries are those of Table~\ref{tab:measurement-summary}; complete FP32/BF16 and memory comparisons appear in Table~\ref{tab:thresholds}.}
\label{tab:main-results}
\small
\setlength{\tabcolsep}{5pt}
\begin{tabular}{@{}lllrrl@{}}
\toprule
Problem & Encoder & Precision & J/instance & AET ($10^3$) & 95\% interval \\
\midrule
\multicolumn{6}{l}{\emph{Windows workstation}} \\
CVRP50 & AM & FP32 & 6.60 & 135.5 & [133.8, 137.0] \\
 & GNN & FP32 & 6.66 & $146.1^{\dagger}$ & [144.1, 148.2] \\
CVRP100 & AM & FP32 & 45.52 & NS & n/a \\
 & GNN & FP32 & 46.43 & NS & n/a \\
 & AM & BF16 & 14.09 & 1,053.4 & [1,045.1, 1,063.2] \\
 & GNN & BF16 & 14.43 & 1,133.5 & [1,122.7, 1,144.2] \\
\midrule
\multicolumn{6}{l}{\emph{GH200 module}} \\
CVRP50 & AM & FP32 & 5.05 & 23.2 & [22.9, 23.6] \\
 & GNN & FP32 & 5.09 & $23.4^{\dagger}$ & [22.0, 24.6] \\
CVRP100 & AM & FP32 & 32.34 & $140.0^{\dagger}$ & [139.1, 141.4] \\
 & GNN & FP32 & 32.18 & $142.1^{\dagger}$ & [138.1, 144.5] \\
\bottomrule
\end{tabular}
\end{table}

Batch size changes the marginal energy entering AET (Figure~\ref{fig:batch}). On GH200, FP32 energy varies little between batches 64 and 512: it is approximately 5 J per instance for CVRP50 and 32 J for CVRP100. Once this plateau is reached, a larger batch changes the threshold little. Windows shows a less regular curve. GNN/CVRP100 at batch 96 has a mean near 910 J per instance despite its recorded memory peaks remaining below capacity. We retain that observation and its interval; Appendix~\ref{app:memory} reports its round-level values and the other memory diagnostics.

\begin{figure}[tp]
\centering
\includegraphics[width=\textwidth]{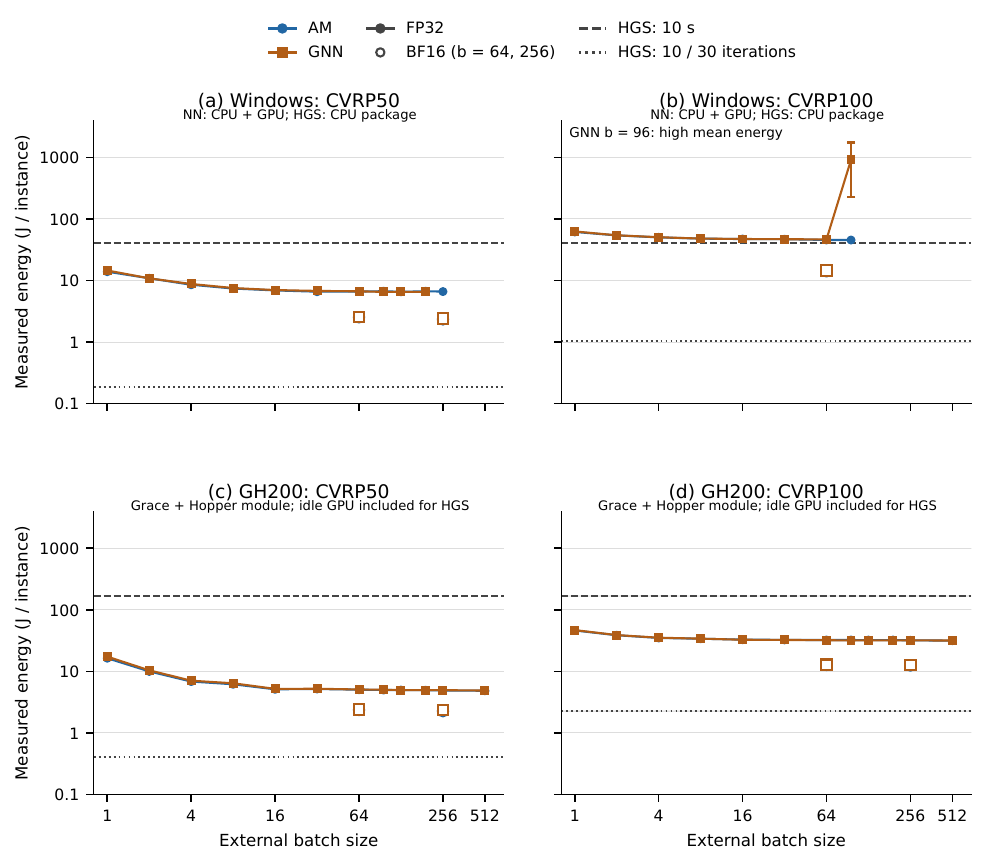}
\caption{POMO $n\times8$ solving energy as a function of batch size. Points average five Windows seed/round pairs or four GH200 training seeds with five rounds each; bars are descriptive 95 percent intervals. Open markers show the two measured BF16 batches. Dashed lines show HGS-10s; dotted lines show short HGS, all with 16 workers. Windows retains policies whose allocated and reserved peaks stay within physical capacity in every round; omitted points break curves. GH200 includes all 56 policies, with separately matched CVRP100 HGS references. Windows counts CPU plus GPU for NN and CPU only for HGS; GH200 counts a module for both, including idle GPU power. Figure~\ref{fig:batch-all} gives the full Windows curves.}
\label{fig:batch}
\end{figure}

Precision changes the CVRP100 comparison at equal batch size. On Windows at batch 64, FP32 POMO $100\times8$ costs 45.52 J for AM and 46.43 J for GNN, both above HGS-10s at 40.74 J. BF16 lowers the neural values to 14.09 and 14.43 J while preserving admissibility, giving thresholds near 1.05 and 1.13 million instances. No tested Windows FP32 batch for this decoding policy has a positive saving against HGS-10s. The conclusion already holds at batch 64, without relying on the larger-than-VRAM measurements. It does not apply to every FP32 decoding rule.

GH200 gives a positive saving for CVRP100 in FP32: the two neural encoders use about 32 J per instance, compared with 168.6 J for their matched HGS-10s series. Their observed-debt thresholds are approximately 140,000 and 142,000 instances, both lower bounds. For CVRP50, AM crosses near 23,200 instances, with an interval of 22,900 to 23,600. These smaller thresholds cannot be attributed to the GPU alone. HGS includes idle Hopper energy in the module measurement; on Windows its primary boundary includes only the CPU package. Appendix~\ref{app:results} gives the complete batch-64 precision comparisons and their quality statistics.

Figure~\ref{fig:budget} shows how these comparisons accumulate over a deployment workload. Neural curves start with the training debt and grow at the measured solving-energy rate. Different admissible batches and precisions produce a range of curves. Their intersections with HGS depend on the heuristic budget as well as neural efficiency. Reducing that budget moves the crossover to a larger volume and can remove the positive marginal saving entirely.

\begin{figure}[p]
\centering
\includegraphics[width=0.97\textwidth]{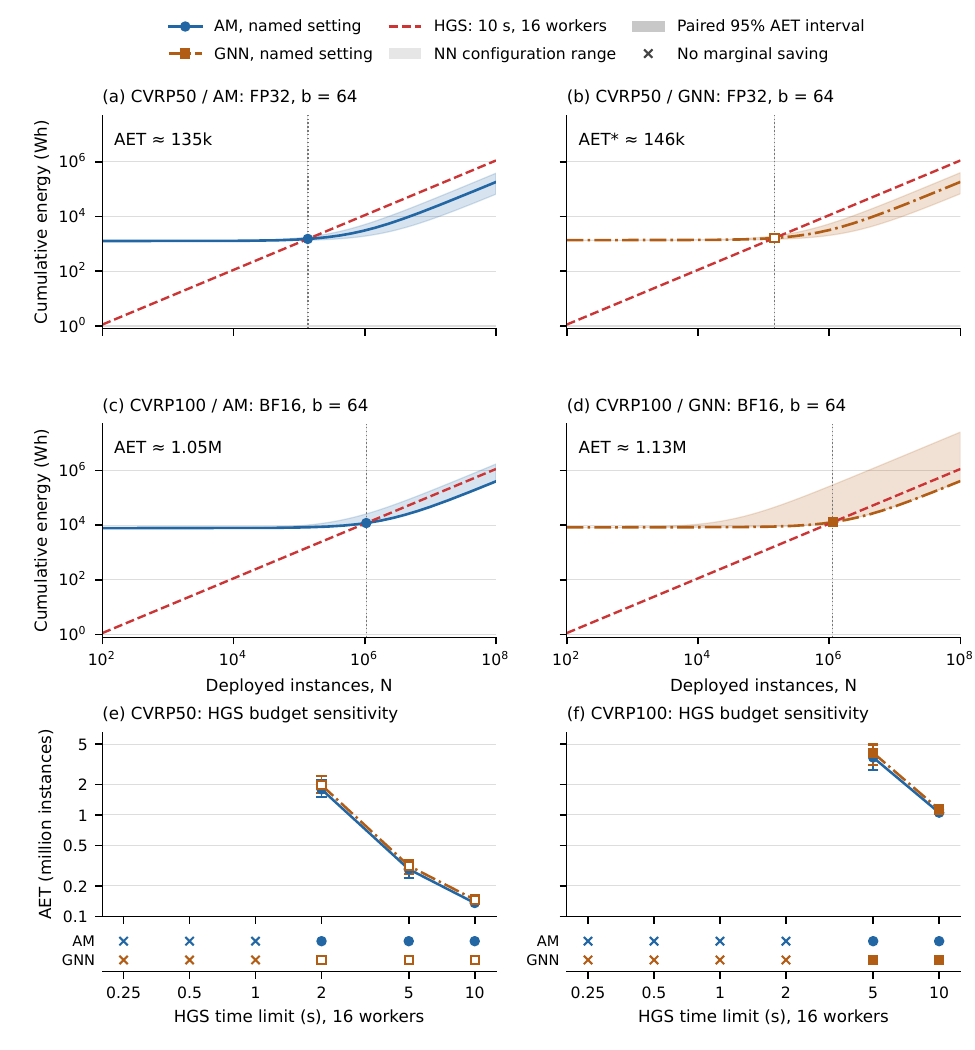}
\caption{Windows cumulative energy and sensitivity to the HGS budget at a 5 percent quality tolerance. Panels (a) to (d) include observed training debt; the focal policies use batch 64, FP32 for CVRP50 and BF16 for CVRP100. Colored envelopes span admissible batch/precision means within recorded memory capacity, not confidence intervals. Gray bands are descriptive threshold intervals; GNN/CVRP50 thresholds remain lower bounds. Panels (e,f) vary the measured HGS time budget, with crosses where no marginal saving exists. Curves use continuous volumes; reported thresholds use full batches. NN includes CPU-package and GPU energy; HGS includes CPU-package energy.}
\label{fig:budget}
\end{figure}

At the 5 percent tolerance, a short HGS policy is a decisive alternative. Windows development selection gives 10 iterations for CVRP50 and 30 for CVRP100, both with 16 workers. They remain admissible on test and use only 0.184 and 1.025 J per instance. Every admissible final neural policy measured on Windows has greater marginal energy, including the policies with fewer POMO augmentations. None can repay a positive training debt against these selected baselines. The predeclared short HGS policies on GH200 also pass the quality gate and use about 0.410 and 2.30 J per instance, below every measured GH200 neural policy. This comparison concerns fixed short policies on GH200, not a newly optimized energy frontier.

The quality requirement changes which conclusions apply. Greedy neural decoding is inadmissible at 5 percent in the workstation examples, even when an energy-only crossover would exist. At a 10 percent tolerance, CVRP100 GNN greedy decoding becomes admissible and six measured batches have positive savings against the development-selected HGS. Its batch-512 threshold is nevertheless about 103 million instances. Appendix~\ref{app:results} reports the full tolerance sweep. The absence of an AET against short HGS at 5 percent is therefore a result for that acceptance rule, not for every possible application.

\section{Interpretation and limits}
\label{sec:discussion}

AET is a property of a comparison. A positive threshold against HGS-10s identifies when replacing that service can repay the declared training cost. It does not establish that the neural policy is the least-energy way to meet the quality requirement. The short-HGS results make this distinction measurable. For deployment, the baseline should represent either the actual service being replaced or an alternative chosen independently of the test results.

The training debt also depends on the model being deployed. The primary results use the final 200-epoch checkpoint. The saved Windows checkpoints show that acceptable quality can occur earlier, at a smaller training debt (Appendix~\ref{app:results}). These are comparisons of observed checkpoints, not an independently validated early-stopping policy. Charging development or model selection to the deployed system would instead increase its debt. If a model is retrained every $R$ instances under unchanged conditions, each cycle saves energy only when $D_a/R<\Delta e_a$. A workload too short to repay the debt before retraining has no sustained advantage.

Energy boundaries are part of the result. The two platforms differ in hardware, software, replication and the components measured. Neither provides wall-plug energy. GH200 HGS-10s occupies about 0.625 seconds of pool time per served instance at approximately 270 W module power, including the idle GPU. Windows excludes GPU energy from the primary HGS boundary. The comparisons cannot isolate a hardware effect. Appendix~\ref{app:accounting} quantifies alternative accounting scenarios without treating them as additional measurements.

The reported intervals describe observed seed and round variation. They do not recover unobserved training energy, bound systematic sensor error or remove unquantified host activity. A daggered threshold remains a lower bound even when its descriptive interval is narrow. Memory peaks likewise do not prove residency or identify the cause of all timing anomalies. We retain the anomalous observations and distinguish the recorded memory regimes rather than infer mechanisms from energy alone.

Finally, the case study covers two generated CVRP sizes and two related neural encoders under a specified training budget. Its quality criterion concerns route cost, and its energy criterion concerns computation. Routes of different lengths may have different downstream vehicle-energy costs, which are outside the study. Treating solutions within the tolerance as interchangeable can be reasonable for scenario evaluation or internal simulation, but requires justification for executed deliveries. Deployment may also involve distribution shift, underfilled batches and latency constraints; those conditions require workload-specific energy measurements rather than reuse of saturated-throughput rates.

\section{Conclusion}
\label{sec:conclusion}

AET connects training energy, batch-dependent solving energy and a shared quality requirement to a concrete deployment volume. The routing case study shows finite break-even volumes against HGS-10s, but no marginal saving against the shorter admissible HGS policies at a 5 percent tolerance. Both outcomes are useful: they specify when a particular replacement can repay training and when a cheaper baseline removes that opportunity. The result must travel with its baseline, training-debt definition and measurement boundary.

Software resources are available in the AET and MLCO repositories~\cite{afifi2026aetsoftware,afifi2026mlco}. The accompanying source bundle contains policy and block summaries, training-debt records, numerical checks and figure-reproduction code. The appendices provide the detailed protocol and sensitivity analyses needed to interpret and reproduce the reported comparisons.

\section*{Acknowledgments}

The GH200 experiments presented in this paper were carried out using the CALCULCO computing platform, supported by DSI/ULCO (Direction des Syst\`emes d'Information de l'Universit\'e du Littoral C\^ote d'Opale).

\clearpage
\appendix

\FloatBarrier
\section{Detailed experimental protocol}
\label{app:protocol}

CVRP instances have independent normalized Euclidean coordinates, integer demands from 1 to 9, and capacities 40 and 50 for 50 and 100 customers, respectively. Each size has a development corpus of 256 instances and a disjoint test corpus of 512. Reference costs are the instancewise best independently validated routes from three HGS runs of 1,000 iterations with seeds 801, 802, and 803. Customer coverage, capacity, depot structure, and Euclidean cost are checked by replaying every returned route. The timed baseline does not define its own gap.

AM has three global self-attention encoder layers, eight heads, hidden dimension 128, and 741,248 trainable parameters. GNN replaces the encoder with three graph-transformer layers, retaining the same pointer decoder. It uses a 10-nearest-neighbor graph augmented with bidirectional depot-customer edges, geometric edge features with 16 radial basis functions, residual connections, dropout 0.1, and 794,880 parameters. Its message-passing design follows the graph-transformer family of Shi et al.~\cite{shi2021unimp}.

On Windows, five seeds per size and architecture, prospectively numbered 2 through 6, are trained for 200 epochs of 20,000 generated instances. POMO training uses $n$ starts for $n$ customers, logical batch 256, microbatch 64 with gradient accumulation, BF16, Adam at learning rate $10^{-4}$, weight decay $10^{-6}$, and gradient clipping at 1. GNN batch normalization operates on the microbatch. The learning rate is multiplied by 0.1 at epochs 180 and 190. The final epoch is the primary checkpoint, without test-based selection. Saved epochs 20, 40, 60, 80, 100, and 150 support a separate workstation debt-quality analysis.

The GH200 evaluation uses four independent training seeds, 2, 3, 4 and 5, for each size and encoder: sixteen models trained from scratch with the same 200-epoch recipe. AM training uses module 0 and GNN training module 1, with CPU affinity and memory bound to the corresponding NUMA node. Deployment uses module 0 for every policy. Thus GNN's debt and deployment energy are measured on different modules; the experiment does not isolate architecture from module differences. Both platforms use identical frozen test corpora and reference files.

On Windows, final-checkpoint inference covers greedy decoding and POMO with $n$ starts and 1, 2, 4, or 8 augmentations. The FP32 test batch sizes are $\{1,2,4,8,16,32,64,96,128,192,256,512\}$. POMO with eight augmentations also has BF16 measurements at batches 64 and 256. Development uses batch 256. Precision is evaluated for both energy and quality; neither precision is assumed equivalent without checking its routes. The Windows CVRP100 POMO $100\times8$ crossover depends on this two-point BF16 grid, although training already uses BF16. Intermediate checkpoints use batch 256 with greedy, POMO $n\times1$, and POMO $n\times8$ policies. GH200 measures final-epoch POMO $n\times8$ over the same twelve FP32 batches and two BF16 batches, without repeating the decoding or checkpoint search.

We record PyTorch peaks for memory allocated to tensors and memory reserved by its caching allocator. These are different quantities~\cite{pytorchcuda}; neither measures page residency or transfer traffic. For the main Windows batch curves and cumulative-energy envelopes, both peaks must remain at or below the GPU's reported physical capacity of 19.99 GiB in all five rounds. This descriptive split uses the memory records and leaves the quality gate and source data unchanged. Batch 64 provides common examples for both sizes and platforms; it is not selected for minimum test energy. Appendix~\ref{app:memory} gives the full Windows curves and the memory regimes outside this subset. GH200 covers the same batch grid on a 144G HBM3e device. Its records do not contain exact capacity in bytes, so the same capacity test is unavailable. The recorded peaks do not establish exclusive device-memory residency.

The Windows HGS family combines time limits of 0.25, 0.5, 1, 2, 5, and 10 seconds, or iteration limits of 10, 30, and 100, with 1, 8, or 16 single-threaded worker processes. One additional control uses a 1 second limit, 16 workers, and fixed distinct logical-CPU affinities. It does not assume that logical processor indices identify performance or efficiency cores. All these policies are evaluated on development and test data. GH200 instead uses two policies fixed before acquisition: HGS-10s and a short iteration limit of 10 for CVRP50 or 30 for CVRP100, all with 16 workers. The short settings match those selected on Windows; they are not a newly optimized GH200 frontier. Process pools are warmed before measurement. Energy is measured around the whole pool, not summed from overlapping worker measurements.

Table~\ref{tab:protocol} gives the measurement boundary. Each metered block repeats full corpus passes until at least 120 seconds have elapsed. A 14,400 second limit applies at corpus-pass boundaries. Energy per instance is obtained from the round-specific ratio $E_r/M_r$, where $M_r$ counts all original instances processed, including repetitions. Windows averages its five seed/round pairs; GH200 first averages five rounds within each seed, then the four seed means. Multiplying this value by $b$ gives equivalent full-batch energy when the corpus is divisible by $b$; it is not a separately instrumented single-batch reading. Wall time is normalized in the same way, yielding inverse aggregate throughput rather than individual request latency.

\begin{table}[t]
\centering
\caption{Hardware and measurement boundaries. Neither platform measures wall-plug energy. The different boundaries prevent interpreting cross-platform ratios as isolated hardware effects.}
\label{tab:protocol}
\small
\begin{tabularx}{\textwidth}{@{}>{\raggedright\arraybackslash}p{0.16\textwidth}>{\raggedright\arraybackslash}X>{\raggedright\arraybackslash}X@{}}
\toprule
Item & Windows workstation & GH200 node \\
\midrule
Host & Windows 11, build 22635; Intel i7-13700, 24 logical CPUs; RTX A4500, 19.99 GiB physical VRAM, 200 W limit & Linux ARM64; two Grace/Hopper modules, 72 CPU cores per module; NVIDIA GH200 144G HBM3e \\
Runtime & Python 3.12.13; PyTorch 2.12.1/CUDA 13.0; PyG 2.8.0.post1; driver 596.71 & Python 3.12.3; PyTorch 2.12.0/CUDA 12.6; PyG 2.8.1; driver 575.57.08 \\
Energy source & Windows EMI CPU-package counter through CodeCarbon 3.3.1; NVML GPU cumulative-energy counter~\cite{nvml2024}; no TDP estimate & Reported module power from Linux hwmon, 50 ms sensor averaging, 100 ms polling; trapezoidal integration; no addition of overlapping NVML energy \\
Primary boundary & CPU package plus GPU for neural work; CPU package only for HGS, representing CPU-only deployment & One Grace/Hopper module for both solvers; idle Hopper included during HGS; not calibrated against an external meter \\
Placement & One workstation, same selected GPU & AM training on module 0, GNN on module 1; all deployment on module 0; CPU and NUMA memory binding \\
Replication & Five training seeds, five deployment rounds & Four training seeds, five deployment rounds per seed \\
Outside scope & DRAM, other workstation components, power-supply losses and cooling & Other module, components outside the module sensor, power-supply losses and cooling \\
Common & \multicolumn{2}{p{0.78\textwidth}}{PyVRP 0.11.3; five rounds with order rotation and shared HGS blocks within rounds; joules at PUE 1; no vehicle energy, manufacture, or carbon conversion} \\
\bottomrule
\end{tabularx}
\end{table}

The Windows evaluation comprises 2,620 cells: 1,880 test cells, 380 metered development cells, and 360 quality-only intermediate-checkpoint development cells. Numerical results are available for 2,610 cells; FP32 POMO $100\times8$ at batch 512 exceeds the available GPU memory for both architectures across all five seeds. The larger-than-VRAM regimes measured at other batch sizes are reported in Appendix~\ref{app:memory}. GH200 evaluates 56 neural policies across four training seeds and five measurement rounds, with matched measurements of the two HGS policies. The compact data provide the block summaries and Windows development aggregates.

\FloatBarrier
\section{Accounting and statistical details}
\label{app:statistics}

\subsection{Quality requirements and baseline selection}

Quality is evaluated against a locked best-known reference, constructed independently of the timed baseline. For seed $s$ and instance $i$, define
\begin{equation}
 g_{s,i}=100\frac{c_{s,i}-c_i^{\mathrm{ref}}}{c_i^{\mathrm{ref}}}.
 \label{eq:gap}
\end{equation}
The reference is not an optimality certificate; a negative gap is possible. Let $q(a)$ be the maximum of the worst observed mean gap and empirical 95th percentile across training seeds and their executions, and the descriptive one-sided 95 percent bootstrap upper bounds for the mean and 95th percentile. Repeated executions are grouped within their training seed. The empirical quality gate $\mathcal Q_\tau(a)$ requires valid routes, finite costs, and $q(a)<\tau$. The same test applies to $h$.

Quality statistics use the first full corpus pass, with routes from every repetition checked for feasibility. A crossed seed-by-instance bootstrap with 2,000 replicates and seed 42733 estimates descriptive one-sided mean and tail bounds. Resampling uses the five Windows or four GH200 training seeds and the 512 shared test-instance identities. Measurement repetitions and numerical execution variants remain grouped within their training seed. The declared tolerances are 1, 2, 3, 3.5, 4, 5, 7.5, and 10 percent. Five percent is an illustrative acceptance rule, not a universal operational standard. For GH200, a sensitivity check uses each seed's worst recorded gap per instance across corpus cycles and rounds. Every measured GH200 policy remains admissible at 5 percent under this check.

For a declared heuristic family $\mathcal H$, we also select
\begin{equation}
 h^*(\tau)=\underset{h\in\mathcal H:\,\mathcal Q_\tau^{\mathrm{dev}}(h)}{\arg\min}
 \widehat e_h^{\mathrm{dev}}.
 \label{eq:selected}
\end{equation}
The selected policy is frozen before test evaluation and must still pass the test quality gate. It is the least-energy admissible setting observed on development within this family, not a globally optimal heuristic. Fixed-baseline sensitivity results do not replace this selection after seeing test outcomes.

\subsection{Training-debt scope and batch accounting}

Charging all training replications rather than one model multiplies the continuous threshold by five on Windows or four on GH200. The retraining condition in Section~\ref{sec:discussion} assumes unchanged distribution, policy and energy rates in each cycle. Hashing and accounting writes lie outside the training measurement window. Windows optimization and inference measurements require positive GPU-counter increments.

For a workload with underfilled batches, the energy function must describe that workload. Our measurement blocks repeat a corpus of 512 instances. At $b=96$ or $192$, each corpus ends with a smaller batch. For these settings, we divide the energy of a full corpus cycle by 512 and round the crossover to a multiple of 512, rather than interpret it as a full-batch result. This does not characterize arbitrary arrival patterns or arbitrary residual batch sizes.

\subsection{Descriptive uncertainty}

AET intervals use 2,000 paired bootstrap draws. On Windows, each sampled index selects a training debt and its corresponding HGS and neural measurement round. On GH200, an outer draw resamples four training seeds and an independent inner draw resamples the five paired HGS/neural rounds within each sampled seed. The numerator is the sampled mean debt of one model; the denominator is the sampled mean paired saving. HGS blocks are shared across batch settings within each matched campaign. The CVRP100 encoders have separately matched HGS series, whose means are shown individually. Intervals describe variation conditional on observed training energy; they do not cover counter calibration error, unquantified host activity, unobserved training energy, or simultaneous uncertainty across the grid.

Five Windows and four GH200 training seeds provide limited uncertainty resolution. The five GH200 measurement rounds per seed distinguish repeated-execution variation from training-seed variation. Reusing HGS blocks correlates comparisons across batch settings, and high-energy rounds can dominate a mean. The paired bootstrap retains that dependence but provides neither a simultaneous guarantee over the grid nor protection against selecting favorable test minima. The fixed final checkpoint, Windows development-only HGS selection, predeclared GH200 baselines and full reported policy families restrict such selection.

\FloatBarrier
\section{Supplementary results}
\label{app:results}

\subsection{Training debt and focal policies}

Table~\ref{tab:debt} gives the full one-model debt record. CVRP100 requires about six times the CVRP50 training energy on Windows and five times on GH200. The similar GH200 CVRP50 means do not identify an architecture effect because training uses different modules and the estimates are uncertain.

\begin{table}[t]
\centering
\caption{Measured training debt and duration per model. Windows rows average five seeds; GH200 rows average four seeds. A dagger denotes a lower bound on training energy and duration. GH200 uses module energy; Windows uses CPU-package plus GPU energy.}
\label{tab:debt}
\begin{tabular}{@{}lllrr@{}}
\toprule
Platform & Problem & Encoder & Debt (Wh) & Duration (h) \\
\midrule
Windows & CVRP50 & AM & 1,283 & 5.44 \\
 & CVRP50 & GNN & $1,381^{\dagger}$ & $5.93^{\dagger}$ \\
 & CVRP100 & AM & 7,797 & 34.87 \\
 & CVRP100 & GNN & 8,282 & 37.37 \\
\midrule
GH200 & CVRP50 & AM & 1,056 & 3.27 \\
 & CVRP50 & GNN & $1,065^{\dagger}$ & $3.20^{\dagger}$ \\
 & CVRP100 & AM & $5,300^{\dagger}$ & $10.88^{\dagger}$ \\
 & CVRP100 & GNN & $5,384^{\dagger}$ & $11.39^{\dagger}$ \\
\bottomrule
\end{tabular}
\end{table}

Table~\ref{tab:quality} gives workstation quality, energy, and inverse throughput for the focal policies. Greedy decoding requires tolerances of 11.48 and 10.71 percent for AM and GNN on CVRP50, and 10.85 and 9.73 percent on CVRP100. Its cheap energy crossovers therefore do not establish AET at 5 percent. Required tolerance can change with precision; FP32 quality is not assigned to BF16 runs.

\begin{table}[t]
\centering
\caption{Windows test quality and component-energy means. All HGS rows use 16 workers. $q$ is the required tolerance, including worst-seed and bootstrap bounds. Time is block wall time per processed instance, not response latency.}
\label{tab:quality}
\small
\begin{tabular}{@{}llrrrr@{}}
\toprule
Problem & Policy & Mean gap (\%) & $q$ (\%) & J/instance & ms/instance \\
\midrule
CVRP50 & AM, FP32, $b=64$, $50\times8$ & 1.78 & 3.49 & 6.60 & 30.3 \\
 & GNN, FP32, $b=64$, $50\times8$ & 1.53 & 3.07 & 6.66 & 29.7 \\
 & HGS-10s & -0.003 & 0.00 & 40.69 & 625.6 \\
 & HGS-10 iterations & 1.96 & 4.34 & 0.184 & 2.36 \\
\midrule
CVRP100 & AM, BF16, $b=64$, $100\times8$ & 3.23 & 4.68 & 14.09 & 62.9 \\
 & GNN, BF16, $b=64$, $100\times8$ & 2.50 & 3.80 & 14.43 & 64.3 \\
 & HGS-10s & 0.014 & 0.46 & 40.74 & 627.3 \\
 & HGS-30 iterations & 2.88 & 4.56 & 1.025 & 14.09 \\
\bottomrule
\end{tabular}
\end{table}

\subsection{Batch-64 comparisons and decoding effort}

Table~\ref{tab:thresholds} reports both precisions on both platforms. Computed from unrounded measurements, the full-batch thresholds are 135,488 instances for AM and $146,112^{\dagger}$ for GNN in CVRP50 FP32, and 1,053,376 and 1,133,504 for AM and GNN in CVRP100 BF16.

\begin{table}[t]
\centering
\caption{Batch-64 comparisons against HGS-10s with 16 workers at $\tau=5\%$, using POMO $n\times8$ and epoch 200. Reserved memory is the largest allocator peak across the observations for that policy. Volumes and paired 95 percent descriptive intervals are in thousands of original instances. NS denotes no marginal saving. A dagger marks a lower-bound debt and threshold; the interval is conditional on observed debt. GH200 intervals resample training seeds and paired measurement rounds separately. Energy boundaries differ as in Table~\ref{tab:protocol}.}
\label{tab:thresholds}
\small
\setlength{\tabcolsep}{4pt}
\begin{tabular}{@{}lllrrll@{}}
\toprule
Problem & Encoder & Precision & J/instance & Reserved (GiB) & AET ($10^3$) & 95\% interval \\
\midrule
\multicolumn{7}{l}{\emph{Windows: CPU package + GPU for NN; CPU package for HGS}} \\
CVRP50 & AM & FP32 & 6.60 & 2.82 & 135.5 & [133.8, 137.0] \\
 & GNN & FP32 & 6.66 & 4.00 & $146.1^{\dagger}$ & [144.1, 148.2] \\
 & AM & BF16 & 2.52 & 1.50 & 121.0 & [120.1, 121.9] \\
 & GNN & BF16 & 2.56 & 2.86 & $130.4^{\dagger}$ & [128.9, 131.9] \\
\midrule
CVRP100 & AM & FP32 & 45.52 & 10.84 & NS & n/a \\
 & GNN & FP32 & 46.43 & 12.87 & NS & n/a \\
 & AM & BF16 & 14.09 & 5.62 & 1,053.4 & [1,045.1, 1,063.2] \\
 & GNN & BF16 & 14.43 & 9.37 & 1,133.5 & [1,122.7, 1,144.2] \\
\midrule
\multicolumn{7}{l}{\emph{GH200: Grace/Hopper module for both solvers; four training seeds}} \\
CVRP50 & AM & FP32 & 5.05 & 2.79 & 23.2 & [22.9, 23.6] \\
 & GNN & FP32 & 5.09 & 4.01 & $23.4^{\dagger}$ & [22.0, 24.6] \\
 & AM & BF16 & 2.37 & 1.50 & 22.8 & [22.5, 23.2] \\
 & GNN & BF16 & 2.41 & 2.87 & $23.0^{\dagger}$ & [21.7, 24.2] \\
\midrule
CVRP100 & AM & FP32 & 32.34 & 10.83 & $140.0^{\dagger}$ & [139.1, 141.4] \\
 & GNN & FP32 & 32.18 & 12.86 & $142.1^{\dagger}$ & [138.1, 144.5] \\
 & AM & BF16 & 12.75 & 5.61 & $122.4^{\dagger}$ & [121.6, 123.6] \\
 & GNN & BF16 & 12.85 & 9.36 & $124.5^{\dagger}$ & [121.2, 126.4] \\
\bottomrule
\end{tabular}
\end{table}

For GH200 CVRP100 at batch 64, the separately matched HGS-10s series use about 168.6 J per instance. Required tolerances are 4.57 and 3.83 percent for AM and GNN in FP32, and 4.64 and 3.80 percent in BF16. Block wall time per instance falls from about 60.5 ms in FP32 to 30.4 ms in BF16. For GH200 CVRP50 FP32, HGS-10s uses 169.03 J per instance; AM and GNN block times are 10.71 and 10.81 ms, mean gaps 1.78 and 1.53 percent, and required tolerances 3.44 and 3.07 percent. The largest reserved peak over the GH200 grid is 102.14 GiB; it does not establish memory residency.

The predeclared short GH200 policies require tolerances of 4.48 and 4.68 percent, use about 0.410 and 2.30 J per instance, and take 1.49 and 8.32 ms per processed instance for CVRP50 and CVRP100, respectively. All 56 measured neural policies have a positive saving against HGS-10s and none against short HGS.

Among admissible final Windows policies at 5 percent, the observed test-energy minima are 1.59 J for AM and 1.64 J for GNN on CVRP50, both using FP32, two POMO augmentations, and batch 512. On CVRP100 they are 14.09 J for AM with BF16, eight augmentations, and batch 64, and 11.64 J for GNN with FP32, two augmentations, and batch 256. These descriptive test minima are not independently validated policy selections and still exceed the selected HGS energies.

\subsection{Quality-tolerance sensitivity}

Table~\ref{tab:tolerance} retains the full declared Windows sweep. The HGS policy selected on development fails the test gate at CVRP50/1 percent and CVRP100/2 and 3.5 percent; no replacement is selected from test results. At 10 percent, CVRP100 GNN greedy becomes admissible, with six measured batches showing a positive saving against selected HGS. At batch 512 it uses 0.119 J per instance versus 0.407 J for HGS, but repaying the training debt takes about 103 million instances, with a descriptive interval of 96 to 109 million. Thus the absence of finite AET at 5 percent does not extend to every tolerance.

\begin{table}[H]
\centering
\caption{Windows tolerance sensitivity for final-checkpoint policies. Counts are admissible neural configurations, not independent replications. Energy is measured for 124 CVRP50 configurations and 122 CVRP100 configurations; two CVRP100 configurations exceed the available GPU memory. F denotes at least one finite conditional threshold against the development-selected HGS; E denotes no positive saving among admissible neural policies; Q denotes no admissible neural policy. An asterisk also indicates that the selected HGS fails the test gate.}
\label{tab:tolerance}
\small
\begin{tabular}{@{}r|rr|rr@{}}
\toprule
& \multicolumn{2}{c|}{CVRP50} & \multicolumn{2}{c}{CVRP100} \\
$\tau$ (\%) & Admissible & Outcome & Admissible & Outcome \\
\midrule
1 & 0 & Q* & 0 & Q \\
2 & 0 & Q & 0 & Q* \\
3 & 0 & Q & 0 & Q \\
3.5 & 28 & E & 0 & Q* \\
4 & 52 & E & 13 & E \\
5 & 76 & E & 50 & E \\
7.5 & 100 & E & 98 & E \\
10 & 100 & E & 110 & F \\
\bottomrule
\end{tabular}
\end{table}

\subsection{Saved-checkpoint debt and quality}

The saved Windows checkpoints expose a debt-quality tradeoff that final-epoch comparisons conceal (Figure~\ref{fig:checkpoints}). With FP32 POMO $n\times8$ at batch 256, the first saved checkpoints passing 5 percent are epoch 60 for AM/CVRP50 and epoch 40 for GNN/CVRP50, at prefix debts of 393 and 279 Wh. On CVRP100, they are epochs 200 and 100, at 7,797 and 4,156 Wh. These are first passing \emph{saved} checkpoints, not proofs of the earliest admissible epoch. The plot describes route quality and training debt; its CVRP100 evaluations exceed physical VRAM and do not establish deployment efficiency within VRAM. GNN reaches the observed gate with less training energy in this comparison, even though its final 200-epoch debt is larger. The prefix excludes subsequent training and the cost of selecting among checkpoints. It is a counterfactual early-stopping cost, not a retrospective reduction of actual study expenditure, and must be paired with the quality and inference energy of that same checkpoint.

\begin{figure}[t]
\centering
\includegraphics[width=\textwidth]{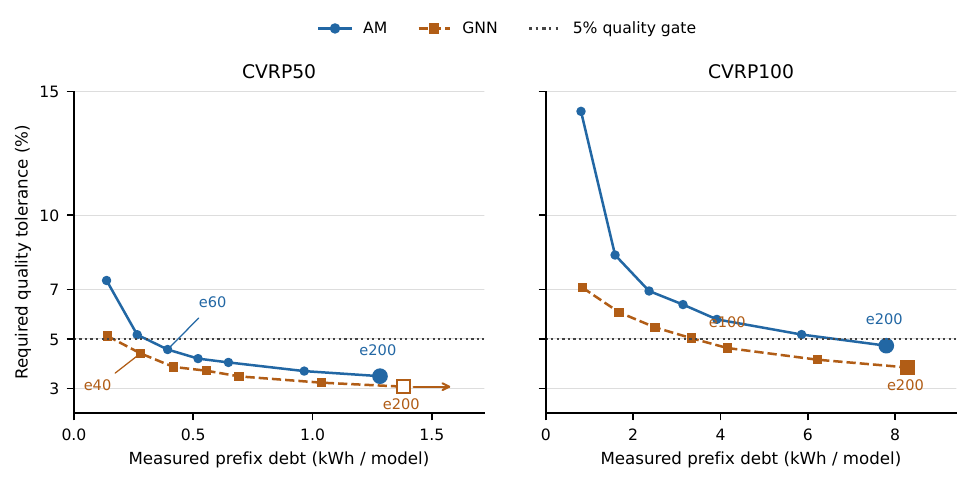}
\caption{Windows required quality tolerance versus observed one-model training debt at the predeclared checkpoints, using FP32 POMO $n\times8$ at batch 256. Labels identify epochs. Passing the illustrated gate requires a value strictly below 5 percent. This plot does not report deployment energy: the CVRP100 checkpoint evaluations exceed physical VRAM. The horizontal energy ranges differ between sizes. The final GNN/CVRP50 debt is a lower bound, marked by a dagger in Table~\ref{tab:debt}.}
\label{fig:checkpoints}
\end{figure}

\FloatBarrier
\section{Accounting sensitivity}
\label{app:accounting}

\subsection{Measurement boundaries and omitted power}

A same-host CPU-plus-GPU comparison and omitted-power scenarios are supplied with the numerical results. For an added constant power $P$, these use
\begin{equation}
 D_a(P)=D_a+P T_a,\qquad
 \Delta e_a(P;h)=e_h-e_a+P(t_h-t_a),
 \label{eq:power}
\end{equation}
where $T_a$ is measured training duration and $t_h,t_a$ are block wall time per processed instance. Values of 10, 25, and 50 W are scenarios, not wall-meter readings. They change both the numerator and denominator. For example, adding 25 W changes the AM/CVRP50 batch-64 FP32 threshold from about 135,000 to 104,000 instances because HGS occupies the host longer per served instance. Counting the measured GPU energy during HGS instead gives about 121,000 instances. These separate scenarios must not be interpreted as a whole-system measurement.

Unquantified host activity may contribute to the Windows measurements. Its effect cannot be separated from solver energy in those records, and the descriptive intervals do not bound it. GH200 deployment windows are disjoint.

\subsection{Incomplete debt and deployment assumptions}

For a lower-bound debt based on $S$ training seeds, let $U\geq0$ denote the total unobserved training energy. The one-model mean debt is $D_a^{\mathrm{obs}}+U/S$, so the continuous crossover increases by $U/(S\Delta e_a)$ when the saving is positive. No finite upper bound on $U$ is established by the measurements. This accounting uncertainty is separate from the descriptive intervals for observed seed and round variation.

The experiments are two generated CVRP sizes, not a claim about all combinatorial optimization. Their 4 million training instances per model are a specified budget, not evidence that either architecture has reached its best attainable quality. Real deployment may involve distribution shift, retraining, underfilled batches, and queueing constraints. With arrival rate $\lambda$ and a batching wait budget $L$, an idealized steady-arrival model limits batches to approximately $b\leq 1+\lambda L$ before service time is considered. A deployment analysis must account for the resulting batch occupancy and idle energy rather than reuse a saturated-throughput threshold unchanged.

These results do not establish that neural routing becomes environmentally preferable after a stated number of deliveries. Carbon comparisons would additionally require grid and hardware-lifecycle assumptions on both sides.

\FloatBarrier
\section{Memory regimes and batch measurements}
\label{app:memory}

The Windows memory audit compares each tensor-allocation and allocator-reservation peak with the reported capacity of 21,464,350,720 bytes (19.99 GiB). Table~\ref{tab:memory} lists the batch-256 cases separated from the main examples. All four CVRP100 configurations exceed physical capacity in tensor allocations in every round, including BF16. Their energy and quality values describe this larger-than-VRAM regime. The AM/BF16 batch-256 crossover of 2.09 million instances belongs to that regime.

\begin{table}[t]
\centering
\caption{Windows batch-256 memory diagnostics for final POMO $n\times8$ policies. Energy is the mean of all five rounds. Memory columns show the maximum peak across rounds, in GiB, against 19.99 GiB physical capacity. The last column counts rounds in which either peak exceeds capacity. The GNN/CVRP50 case exceeds capacity in reservation only; all CVRP100 rows exceed it in tensor allocation as well.}
\label{tab:memory}
\small
\setlength{\tabcolsep}{5pt}
\begin{tabular}{@{}lllrrrr@{}}
\toprule
Problem & Encoder & Precision & J/instance & Allocated & Reserved & Rounds \\
\midrule
CVRP50 & GNN & FP32 & 9.21 & 10.94 & 57.77 & 1/5 \\
CVRP100 & AM & FP32 & 175.69 & 42.23 & 43.00 & 5/5 \\
 & GNN & FP32 & 493.19 & 42.23 & 58.28 & 5/5 \\
 & AM & BF16 & 27.33 & 21.99 & 22.41 & 5/5 \\
 & GNN & BF16 & 95.48 & 29.97 & 58.01 & 5/5 \\
\bottomrule
\end{tabular}
\end{table}

The five GNN/CVRP50 batch-256 energies are approximately 6.56, 6.55, 6.54, 19.93, and 6.45 J per instance. The high-energy round reserves 57.77 GiB, compared with about 15.84 GiB for each other round, while all five tensor-allocation peaks stay near 10.94 GiB. It takes 100.8 ms per instance, compared with about 28.6 ms in the other rounds. Its cause is not identified by the memory peaks. All five rounds enter the paired AET interval of 143,900 to 192,000 instances, conditional on the lower-bound training debt.

GNN/CVRP100 at batch 96 has a mean energy of 909.64 J per instance and reserved peaks of 19.23 to 19.30 GiB, yet three rounds consume approximately 945, 2,571, and 938 J per instance, compared with 47.1 and 46.6 J in the other two. This point remains in the main batch curve because its recorded peaks do not exceed capacity. Allocator peaks therefore neither certify residency nor explain all timing variation. Distinguishing the mechanisms would require device and operating-system traces with controlled system-memory use.

\begin{figure}[t]
\centering
\includegraphics[width=\textwidth]{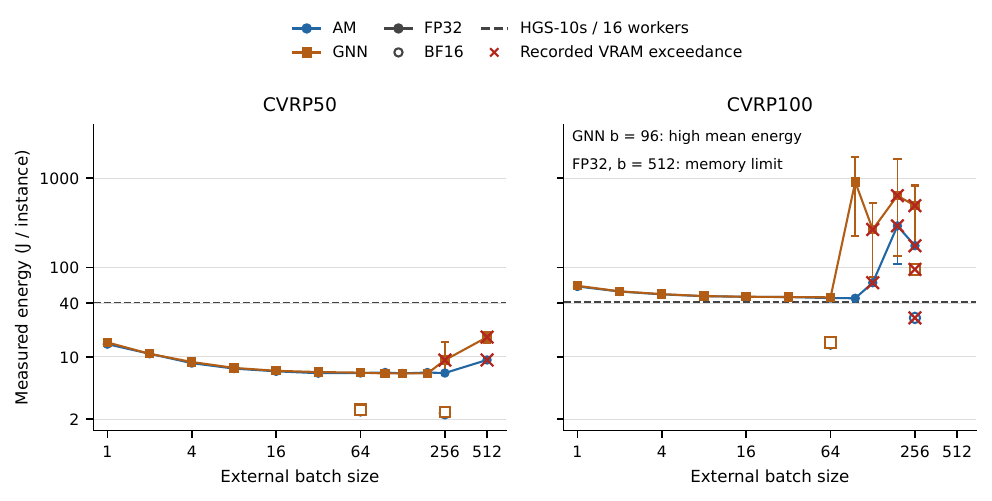}
\caption{Windows final-checkpoint POMO $n\times8$ batch sweep, including configurations omitted from Figure~\ref{fig:batch} because at least one allocated or reserved peak exceeds physical VRAM. Marked points retain all five rounds and their descriptive 95 percent intervals. BF16 has only two measured batch sizes. CVRP100 FP32 at batch 512 exceeds the available GPU memory and has no energy point. The machine-readable memory audit supplies both peak measures and exceedance counts for each policy.}
\label{fig:batch-all}
\end{figure}

Scheduling, CPU power limits, and memory policy can affect the Windows measurements. Tensor allocations above physical VRAM exclude an interpretation as execution entirely within device memory. NVIDIA documents a Windows mechanism that uses system memory near GPU capacity~\cite{nvidiasysmem}; the observed allocations support that interpretation without quantifying paging or PCIe transfers. The measurements do not isolate the effects of CPU power settings or the driver policy, and their boundary excludes DRAM energy. Fixed logical affinity does not isolate core type or thermal history. GH200 uses a different memory system and software stack, so its large-batch results do not identify the mechanism behind the Windows behavior.

\FloatBarrier
\section{Numerical reproduction}
\label{app:reproducibility}

The numerical checks cover 4,000 Windows training epochs and all 1,880 test cells. Independent reconstruction of energy normalization and quality statistics agrees with the recorded policy aggregates. Memory records include the physical-capacity source, per-policy peak ranges, and code to reconstruct the capacity comparison. The provenance supplement provides the accounting ledger and source-identity checks.

The frozen development and test corpora, together with their reference costs and routes, are included in the accompanying data. These reference solutions are best-known solutions rather than certified optima. The result files retain unrounded values and policy identifiers. Conditional intervals use the 2.5th and 97.5th percentile order statistics of 2,000 paired resamples. A nonpositive resampled saving gives an infinite threshold, so an interval may have no finite upper endpoint. Quality eligibility is the empirical gate in Section~\ref{sec:framework} and is held fixed within the energy bootstrap.

The GH200 checks reconstruct power integration and verify block identities and hashes. Independent replay checks 619,520 saved first-pass deployment routes and 1,024 reference routes for customer coverage, capacity and Euclidean cost, with a maximum cost discrepancy below $2.5\times10^{-14}$. Subsequent corpus passes provide cost and online feasibility records rather than routes for independent replay.

The compact GH200 data provide block ratios, per-instance quality vectors, training-energy summaries, 56 neural policy summaries, separately matched HGS series, and 112 neural-baseline comparisons. A numerical checker reconstructs quality gates, seed/round intervals and threshold rounding from these extracts. Raw power traces and route logs are needed to repeat sensor integration and route replay. The energy values describe reported module power; they are not externally calibrated whole-system measurements.

\FloatBarrier
\clearpage
\small
\bibliographystyle{plain}
\bibliography{references}
\end{document}